\documentclass[11pt]{article}
\usepackage[left=1in, right=1in, top=1in, bottom=1in]{geometry}
\pdfoutput=1

\usepackage{microtype}
\usepackage{graphicx}
\usepackage{subfigure}
\usepackage{booktabs} 

\usepackage{hyperref}
\graphicspath{{.}{fig/}}
\DeclareGraphicsExtensions{.pdf, .png, .jpg}

\usepackage{amsmath, amsfonts, amsthm, amssymb}
\usepackage{dsfont} 
\usepackage[numbers]{natbib}
\usepackage{algorithm}
\usepackage{algorithmic}
\newcommand{\IFTHENELSE}[3]{
	\algorithmicif\ #1 \algorithmicthen\ #2 \algorithmicelse\ #3}
\newcommand{\SWITCH}[1]{\STATE \textbf{switch} #1 \textbf{do}}
\newcommand{\ENDSWITCH}{\STATE \textbf{end switch}}
\newcommand{\CASE}[1]{\STATE \textbf{case} #1\textbf{:} \begin{ALC@g}}
\newcommand{\ENDCASE}{\end{ALC@g}}

\newcommand{\DEFAULT}{\STATE \textbf{default:} \begin{ALC@g}}
\newcommand{\ENDDEFAULT}{\end{ALC@g}}
\newcommand{\DEFAULTLINE}[1]{\STATE \textbf{default:} }

\newcommand{\OMIT}[1]{}
\newcommand{\br}[1]{\left(#1\right)}

\newcommand{\cbr}[1]{\left\{#1 \right\}}

\newcommand{\innerprod}[1]{\left\langle #1 \right\rangle}

\newcommand{\pix}{\pi_x}	
\newcommand{\Pix}{\Pi_x}	

\newcommand{\oneh}{\mathds{1}}	

\newcommand{\Sn}{\mathbb{S}_n}	
\newcommand{\RR}{\mathbb{R}}	
\newcommand{\NN}{\mathbb{N}}	
\newcommand{\Ical}[1]{\mathfrak{I}^{(#1)}}	

\newcommand{\T}[1][\sigma]{\Pi_{#1}^{\otimes d}}	
\newcommand{\UU}{\mathcal{U}}	
\newcommand{\BB}{\mathcal{B}}	

\newcommand{\tsorder}[1]{%
  {\everymath{\scriptstyle}%
   \overbrace{\scriptstyle n \times \cdots \times n}^{#1}}%
}
\newcommand{\tsprod}[1][d]{[1, \dots, #1; 1, \dots, #1]}
\newcommand{\tsprodalt}[1][d]{[#1, \dots, 2#1; 1, \dots, #1]}

\newcommand{\idx}[3][d]{#2_1 #3 \dots #3 #2_{#1}}		
\newcommand{\idpi}[3][d]{#3(#2_1), \dots, #3(#2_{#1})}	

\newcommand{\avg}{\textrm{avg}}

\newcommand{\add}{\textrm{add}}
\newcommand{\mult}{\textrm{mult}}
\newcommand{\topk}{@k}
\newcommand{\F}{\textrm{F}}

\newcommand{\vv}{\operatorname{vec}}
\newcommand{\tr}{\operatorname{tr}}

\newcommand{\pluseq}{\mathrel{ {+}{=} }}

\theoremstyle{plain}
\newtheorem{thm}{Theorem}
\newtheorem{lem}{Lemma}
\newtheorem{cor}{Corollary}

\theoremstyle{definition}

\theoremstyle{remark}
\newtheorem*{rmk}{Remark}

\begin{document}

\title{The Weighted Kendall and High-order Kernels for Permutations}
\author{Yunlong Jiao$^{1,2}$, Jean-Philippe Vert$^{3,4,5}$\\ \\
$^1$Wellcome Centre for Human Genetics, University of Oxford, Oxford OX3 7BN, UK\\
$^2$Department of Statistics, University of Oxford, Oxford OX1 3LB, UK\\
$^3$MINES ParisTech, PSL Research University,\\
CBIO -- Centre for Computational Biology, F-75006 Paris, France\\
$^4$Institut Curie, PSL Research University, INSERM, U900, F-75005 Paris, France\\
$^5$Ecole Normale Sup\'{e}rieure, Department of Mathematics and Applications, \\
CNRS, PSL Research University, F-75005 Paris, France\\\\
Emails: \href{mailto:yjiao@well.ox.ac.uk}{yjiao@well.ox.ac.uk}, \href{mailto:jean-philippe.vert@ens.fr}{jean-philippe.vert@ens.fr}
}
\date{}
\maketitle

\begin{abstract}
We propose new positive definite kernels for permutations. First we introduce a weighted version of the Kendall kernel, which allows to weight unequally the contributions of different item pairs in the permutations depending on their ranks. Like the Kendall kernel, we show that the weighted version is invariant to relabeling of items and can be computed efficiently in $O(n \ln(n))$ operations, where $n$ is the number of items in the permutation. Second, we propose a supervised approach to learn the weights by jointly optimizing them with the function estimated by a kernel machine. Third, while the Kendall kernel considers pairwise comparison between items, we extend it by considering higher-order comparisons among tuples of items and show that the supervised approach of learning the weights can be systematically generalized to higher-order permutation kernels.
\end{abstract}

\section{Introduction}
A permutation is a 1-to-1 mapping from a finite set into itself, and allows to represent mathematically a complete ordering or $n$ items.
We consider the problem of machine learning when data are permutations, which has many applications such as analyzing preferences or votes \citep{Diaconis1988Group, Marden1996Analyzing}, tracking objects \cite{Huang2009Fourier}, or learning robustly from high-dimensional biological data \cite{Geman2004Classifying,Lin2009ordering,Jiao2018Kendall}.

A promising direction to learn over permutations is to first \emph{embed} them to a vector space, i.e., to first represent each permutation $\pi$ by a vector $\Phi(\pi)\in\RR^d$, and then to learn a parametric linear or nonlinear model over $\Phi(\pi)$ using standard machine learning approaches. For example, \citet{Kondor2010Ranking} proposed an embedding to $\RR^{d}$ with $d=n!$, where $n$ is the number of items; however, in spite of computational tricks, this leads to algorithms with $O(n^n)$ complexity which become impractical as soon as we consider more than a few items. Recently, \citet{Jiao2018Kendall} showed that the well-known Kendall $\tau$ correlation defines implicitly an embedding in $d=O(n^2)$ dimensions, in which linear models can be learned with $O(n \ln(n))$ complexity thanks to the so-called kernel trick \cite{Schoelkopf2002Learning,Shawe-Taylor2004Kernel}; they showed promising results on gene expression classification where $n$ is a few thousands. In this paper, we propose to extend this work by tackling several limitations of the Kendall kernel.

First, the Kendall kernel compares two permutations by computing the number of pairs of items which are in the same order in both permutations. When $n$ is large, it may be more interesting to weight differently the contributions of different pairs, e.g., to focus more on the top-ranked items. Many weighted versions of Kendall's $\tau$ correlation have been proposed in the literature, as reviewed below in Section~\ref{sec:review}; however, to our knowledge, none of them is associated to a valid embedding of permutations in a vector space, like the original Kendall correlation. Here we propose such an extension, and show that it inherits the computational benefits from the Kendall kernel while allowing to weight differently the importance of items based on their rank.

Second, we discuss how to choose the position-dependent weights that define this embedding. We propose to see them as parameters of the model, and to optimize them jointly with other parameters during training. This is similar in spirit to the way different kernels are weighted and combined in multiple kernel learning \citep{Lanckriet2004Learning,Bach2004Multiple}, or more generally to the notion of representation learning which became central in recent years in conjunction with neural networks or linear models \citep{Mikolov2013Distributed,LeMorvan2017Supervised}. We show in particular that efficient alternative optimization schemes are possible in the resulting optimization problem.

Third, observing that the features defining the embedding associated to the Kendall kernel are based on pairwise comparisons only, we consider extensions to higher-order comparisons, e.g., the relative ranking among triplets or quadruplets of items. Such features are naturally captured by higher-order representations of the symmetric group \citep{Diaconis1988Group,Kondor2010Ranking}, and we show that both the definition of the embedding and the algorithms to learn the weights can be effortlessly extended to such higher-order scenarios.

We finally present preliminary experimental results highlighting the potential of these new embeddings. Code to reproduce all the experiments in the present paper is available at \url{https://github.com/YunlongJiao/weightedkendall}.

\section{The Kendall kernel}
Let us first fix some notations. Given an integer $n\in\NN$, a \emph{permutation} is a 1-to-1 mapping $\sigma: [1,n] \to [1,n]$ such that $\sigma(i)\neq \sigma(j)$ for $i\neq j$. We denote by $\Sn$ the set of all such permutations. A permutation $\sigma\in\Sn$ can for example represent the preference ordering of a user over $n$ items, in which case $\sigma(i)$ is the rank of item $i$ among all items (e.g., the preferred item is the item $i$ such that $\sigma(i)=1$, namely, $\sigma^{-1}(1)$). Endowed with the composition operation $(\sigma_1 \sigma_2)(i) = \sigma_1 (\sigma_2(i))$, $\Sn$ is a group called the \emph{symmetric group}, of cardinality $n!$. The identity permutation will be denoted $e$.

A positive definite (p.d.) kernel on $\Sn$ is a symmetric function $K:\Sn\times\Sn\rightarrow\RR$ that satisfies, for any $l\in\NN$, $\alpha\in\RR^l$ and $\br{\sigma_1,\ldots,\sigma_l}$ in $\Sn^l$,
$$
\sum_{i=1}^l \sum_{j=1}^l \alpha_i \alpha_j K(\sigma_i,\sigma_j) \geq 0 \,.
$$
Equivalently, a function $K:\Sn\times\Sn\rightarrow\RR$ is a p.d. kernel if and only if there exists a mapping $\Phi: \Sn\rightarrow \RR^d$ (for some $d\in\NN$) such that $K(\sigma,\sigma') = \Phi(\sigma)^\top \Phi(\sigma')$. For example, taking $d=n(n-1)$ and $\Phi_\tau(\sigma) = \br{\oneh_{\sigma(i) < \sigma(j)} }_{1 \leq i \neq j \leq n}$, we see that the \emph{number of concordant pairs} between two permutations is a p.d. kernel, i.e., $\forall \sigma,\sigma'\in\Sn$,
\begin{equation}\label{eq:kendall}
\begin{split}
K_\tau(\sigma,\sigma') & = \Phi_\tau(\sigma)^\top\Phi_\tau(\sigma') \\
& = \sum_{1 \leq i \neq j \leq n} \oneh_{\sigma(i) < \sigma(j)} \oneh_{\sigma'(i) < \sigma'(j)} \,.
\end{split}
\end{equation}
Up to constant shift and scaling (by taking $2 K_{\tau} / \binom{n}{2} - 1$), this is equivalent to the \emph{Kendall kernel} of \citet{Jiao2018Kendall}. To lighten notations, we will simply call $K_\tau$ the Kendall kernel in the rest of this paper.

Besides being p.d., the Kendall kernel has another interesting property: it is \emph{right-invariant}, in the sense that for any $\sigma,\sigma'$ and $\pi\in\Sn$, it holds that
$$
K_{\tau}\br{\sigma,\sigma'} = K_{\tau}\br{\sigma \pi,\sigma' \pi}\,.
$$
Right-invariance implies that the kernel does not change if we relabel the items to be ranked, which is a natural requirement in most ranking problems (e.g., when items are only presented in alphabetical order for no other particular reason). Note that any right-invariant kernel can be rewritten as, $\forall \sigma, \sigma' \in \Sn$,
\begin{equation}\label{eq:kappa}
K(\sigma,\sigma') = K(e,\sigma'\sigma^{-1}) =: \kappa(\sigma'\sigma^{-1})\,,
\end{equation}
where $\kappa:\Sn\rightarrow \RR$. Hence a right-invariant kernel is a semigroup kernel \citep{Berg1984Harmonic}, and we say that a function $\kappa:\Sn\rightarrow \RR$ is p.d. when the kernel $K$ defined by (\ref{eq:kappa}) is p.d. In particular, note that $K$ is symmetric if and only if $\kappa$ satisfies $\kappa(\sigma)=\kappa(\sigma^{-1})$ for any $\sigma\in\Sn$. For example, the p.d. function associated to the Kendall kernel (\ref{eq:kendall}) is:
$$
\kappa_\tau(\sigma) = \sum_{1 \leq i \neq j \leq n} \oneh_{i < j} \oneh_{\sigma(i) < \sigma(j)} = \sum_{1\leq i < j \leq n} \oneh_{\sigma(i) < \sigma(j)}\,.
$$
Denoting $\Ical{2} = \cbr{(i,j) \in [1,n]^2 \,:\, 1 \leq i < j \leq n}$, this can be rewritten as
\begin{equation}\label{eq:order2}
\kappa_\tau(\sigma) = \sum_{(i,j)\in\Ical{2}} \oneh_{\br{\sigma(i) , \sigma(j)} \in \Ical{2}}\,.
\end{equation}
This notation will be convenient in Section~\ref{sec:highorder} when we generalize Kendall to high-order kernels for permutations.

\section{Related work}\label{sec:review}
The Kendall kernel (\ref{eq:kendall}) compares two permutations by assessing how many pairs of items are ranked in the same order. In many cases, however, one may want to weight differently the contributions of different pairs, based for example on their rankings in both permutations. Typically in applications involving preference ranking and information retrieval, one may expect that the relative ordering of items ranked at the bottom of the list is usually less relevant than that at the top. For example, when participants are asked to express their preference over a list of predefined items, it is usually impractical for participants to accurately express their preference towards less preferable items. As another example from information retrieval, when we wish to compare two ranked lists of documents outputted by two search engines, differences towards the top usually matter more as those query results will be eventually presented to the user.

Many authors have proposed weighted versions of the Kendall's $\tau$ correlation coefficient. A classical one, for example, is the \emph{weighted Kendall $\tau$ statistics} studied by \citet{Shieh1998weighted}:
\begin{equation}\label{eq:shieh}
\tau_w(\sigma,\sigma') =  \sum_{1 \leq i \neq j \leq n} w(\sigma(i),\sigma(j)) \oneh_{\sigma(i) < \sigma(j)} \oneh_{\sigma'(i) < \sigma'(j)} \,,
\end{equation}
for some weight function $w:[1,n]^2\rightarrow \RR$. Typical examples of the weight function include
\begin{itemize}
	\item $w(i,j)=1/(j-1)$ if $j \geq 2$ and $0$ otherwise, which gives the \emph{average precision correlation coefficient} \citep{Yilmaz2008new}.
	\item Weights of the form $w(i,j)=w_iw_j$ for some vector $w\in\RR^n$ \citep{Kumar2010Generalized}.
\end{itemize}
While (\ref{eq:shieh}) looks like a weighted version of (\ref{eq:kendall}), it is not symmetric in $\sigma$ and $\sigma'$ (except if $w$ is constant) since the weight of a pair of items only depends on their rankings in $\sigma$, and in particular $\tau_w$ is not a p.d. kernel. To enforce symmetry, \citet{Vigna2015weighted} proposes a weighted correlation of the form
\begin{equation}\label{eq:vigna}
\begin{split}
\tau_w(\sigma,\sigma') =  \sum_{1 \leq i \neq j \leq n} & ( w( \sigma( i ), \sigma( j ) ) + w( \sigma'( i ), \sigma'( j ) ) ) \\
& \times \oneh_{\sigma(i) < \sigma(j)} \oneh_{\sigma'(i) < \sigma'(j)} \,,
\end{split}
\end{equation}
which is symmetric and right-invariant, i.e., is invariant by relabeling of the items. However, (\ref{eq:vigna}) is not a valid p.d. inner product. \citet{Kumar2010Generalized} proposes a weighted correlation of the form:
\begin{equation}\label{eq:kumar}
\begin{split}
& \tau_{w,p} (\sigma, \sigma') = \sum_{1 \leq i \neq j \leq n} w(\sigma(i),\sigma(j)) \\
& \times \frac{p_{\sigma(i)} - p_{\sigma'(i)}}{\sigma(i) - \sigma'(i)}  \frac{p_{\sigma(j)} - p_{\sigma'(j)}}{\sigma(j) - \sigma'(j)} \oneh_{\sigma(i) < \sigma(j)} \oneh_{\sigma'(i) < \sigma'(j)} \,,
\end{split}
\end{equation}
where $w:[1,n]^2\rightarrow \RR$ and $p\in\RR^n$. In particular, the additional set of weights $p$ in (\ref{eq:kumar}) is motivated by cumulatively weighting the cost of swaps between adjacent positions needed to transform $\sigma$ into $\sigma'$. Like (\ref{eq:shieh}), $\tau_{w,p}$ is not symmetric hence not p.d. \citet{Farnoud2014axiomatic} notices that Kendall's $\tau$ induces a Euclidean metric which is the shortest path distance over a Cayley's graph (i.e., the smallest number of swaps between adjacent positions to transform a permutation into another), and proposes to set weights on the edges (i.e., on swaps between adjacent positions) in order to define the shortest path over the weighted graph as a new metric; although a valid metric, it does not induce a valid p.d. kernel in general. In summary, to our knowledge, no existing weighted variant of the Kendall kernel is p.d. and right-invariant.

\section{The weighted Kendall kernel}\label{sec:wken}
The following result provides a generic way to construct a weighted Kendall kernel that is p.d. and right-invariant.
\begin{thm}\label{thm:pdk}
Let $W:\NN^2\times\NN^2\rightarrow\RR$ be a p.d. kernel on $\NN^2$. Then the function $K_W:\Sn\times\Sn\rightarrow\RR$ defined by
\begin{equation}\label{eq:kw}
\begin{split}
K_W(\sigma,\sigma') = \sum_{1 \leq i \neq j \leq n} & W\br{(\sigma(i),\sigma(j)),(\sigma'(i),\sigma'(j))} \\
& \times \oneh_{\sigma(i) < \sigma(j)} \oneh_{\sigma'(i) < \sigma'(j)}
\end{split}
\end{equation}
is a right-invariant p.d. kernel on $\Sn$.
\end{thm}
Theorem~\ref{thm:pdk} is easy to prove by writing explicitly $W$ as an inner product, and deducing from (\ref{eq:kw}) that $K_W$ can also be written as an inner product. Note that $K_W$ is always right-invariant whether or not $W$ is p.d. Note also that there may exist non p.d. weight $W$ leading to p.d. kernels $K_W$; we leave it as an open question to characterize the necessary conditions on the weight function $W$ such that $K_W$ is p.d.

Let us now consider particular cases of the weighted Kendall kernel of the following form:
\begin{cor}\label{cor:pdk}
Let the weights in Theorem~\ref{thm:pdk} take the form $W\br{(a,b),(c,d)} = U_{ab} U_{cd}$ for some matrix $U\in\RR^{n\times n}$, then
\begin{equation}\label{eq:ku}
\begin{split}
K_U(\sigma,\sigma') = \sum_{1 \leq i \neq j \leq n} & U_{\sigma(i),\sigma(j)} U_{\sigma'(i),\sigma'(j)} \\
& \times \oneh_{\sigma(i) < \sigma(j)} \oneh_{\sigma'(i) < \sigma'(j)}
\end{split}
\end{equation}
is a right-invariant p.d. kernel on $\Sn$.
\end{cor}
\begin{rmk}
It is interesting to make explicit how Corollary~\ref{cor:pdk} is derived from Theorem~\ref{thm:pdk}. Let us fix the number of items to rank $n\in\NN$. The conditions in Theorem~\ref{thm:pdk} on $W$ being a kernel over $\NN^2$ can now be relaxed to being a kernel over $[1,n]^2$. With a slight abuse of notation, any p.d. kernel $W$ over finite $[1,n]^2$ is uniquely determined by its full Gram matrix $W$ of size $n^2 \times n^2$, which can always be decomposed into $W = UU^\top$ for some matrix $U$ of size $n^2 \times n^2$ due to matrix $W$ being s.p.d. In other words, this factorization-based definition $W = UU^\top$ is necessary and sufficient under the conditions set in Theorem~\ref{thm:pdk}. It is now easy to see that, if we further assume that Gram matrix $W$ has rank 1 such that $U$ is a vector of size $n^2$, i.e., a matrix of size $n \times n$, Theorem~\ref{thm:pdk} reduces to Corollary~\ref{cor:pdk}.
\end{rmk}

Notably in (\ref{eq:ku}), $U_{ab}$ should not be interpreted as a weight for items $a$ and $b$, but rather as a weight for (those items sent by a permutation to) positions $a$ and $b$. More precisely, if two items are ranked in the same relative order in both permutations, say positions $a<b$ in $\sigma$ and $c<d$ in $\sigma'$, then this pair contributes $U_{ab}U_{cd}$ to the kernel value. Note that if $U_{ij}$ is constant for any $(i,j)$, the weighted Kendall kernel (\ref{eq:ku}) reduces to the standard Kendall kernel (\ref{eq:kendall}).

While $U \in \RR^{n \times n}$ encodes the weights of pairs of positions, it is usually intuitive to start from individual positions. Suppose now we are given a vector $u \in \RR^n$ that encodes our belief of relevance of each position in $[1,n]$, some particularly interesting choices for $U \in \RR^{n \times n}$ are
\begin{itemize}
	\item \textbf{Top-$k$ weight:} For some predetermined $k\in[2,n]$, define $U_{ij} = 1$ if $i \leq k$ and $j \leq k$ and $0$ otherwise.
	\item \textbf{Additive weight:} Given some $u \in \RR^n$, define $U_{ij} = u_{i} + u_{j}$.
	\item \textbf{Multiplicative weight:} Given some $u \in \RR^n$, define $U_{ij} = u_{i} u_{j}$.
\end{itemize}
Each of these choices can be relevant in certain applications. For example, the top-$k$ weight compares two permutations using exclusively the top $k$ ranked items, and checking if they appear at the top of both permutations and in the same relative order. As another example, if one wishes to attenuate the importance of items when their rank increases, the additive or multiplicative weights can be a natural choice when combined with, for example, hyperbolic reduction factor $u_i = 1/(i+1)$ of the average precision correlation \citep{Yilmaz2008new}, or logarithmic reduction factor $u_i = 1/\log_{2}(i+1)$ of the discounted cumulative gain widely used in information retrieval \citep{Manning2008Introduction}. In particular, the top-$k$ weight can be seen as a multiplicative weight with a hard cutoff factor $u_i = 1$ if $i \leq k$ and $0$ otherwise. Figure~\ref{fig:plotweight} illustrates the aforementioned three different types of $u$.

\begin{figure}[t]
\begin{center}
\centerline{\includegraphics[width=0.8\columnwidth]{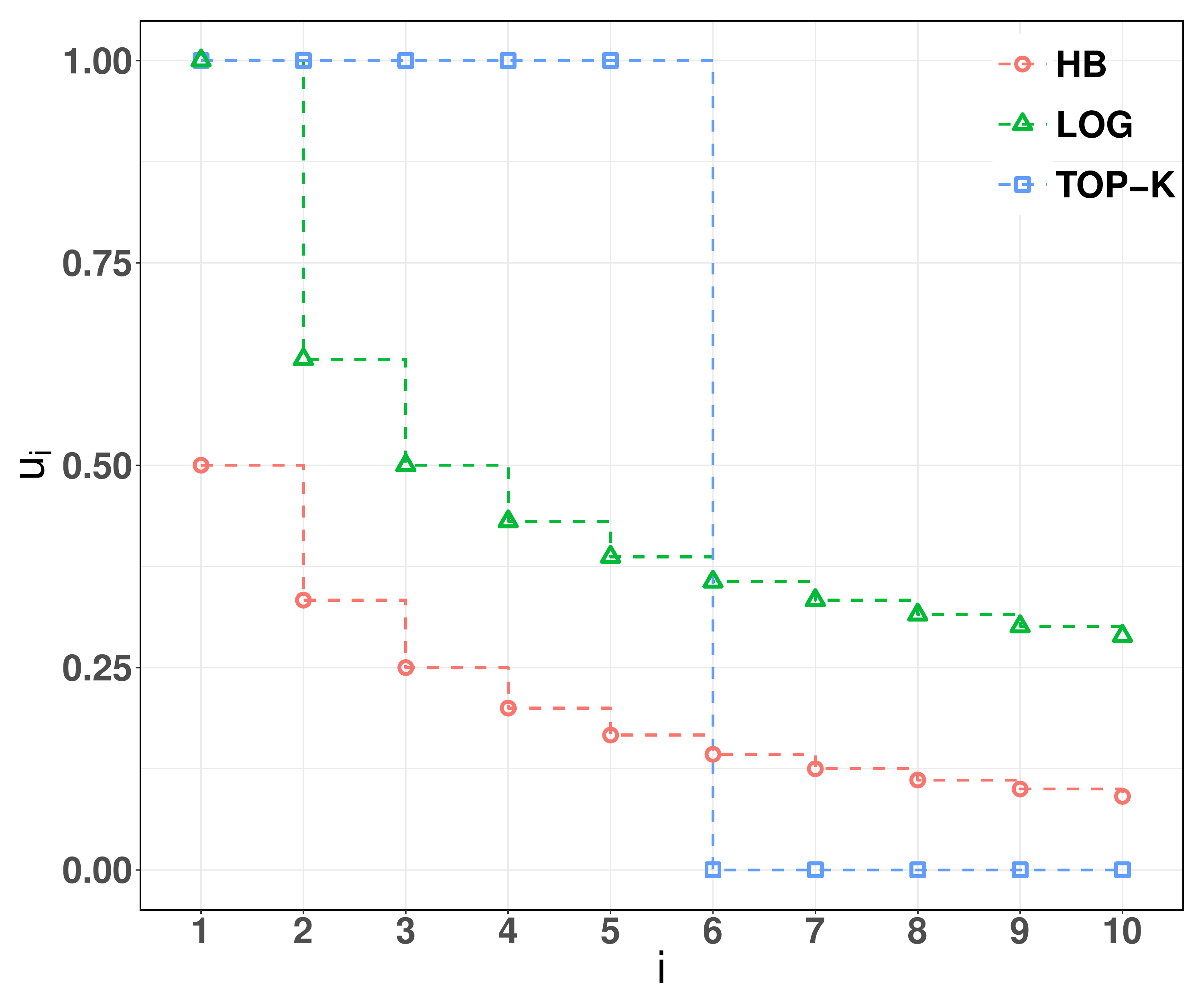}}
\caption{Illustration of hard cutoff for top-$k$ weight, and hyperbolic and logarithmic reduction factors for additive or multiplicative weight.}
\label{fig:plotweight}
\end{center}
\vskip -0.3in
\end{figure}

Note that the hyperparameter $k$ in the top-$k$ Kendall kernel needs to be determined in practice, which can be achieved by cross-validation for instance. A straightforward way to bypass this and construct a kernel agnostic to the particular choice of $k$ is to take the average of all top-$k$ Kendall kernels with $k$ ranging over $[1,n]$. Specifically, the \emph{top-$k$ Kendall kernel} is written
\begin{equation}\label{eq:topk}
\begin{split}
K_U^{\topk}(\sigma, \sigma') = \sum_{1 \leq i \neq j \leq n} & \oneh_{\sigma(i) \leq k} \oneh_{\sigma(j) \leq k} \oneh_{\sigma'(i) \leq k} \oneh_{\sigma'(j) \leq k} \\
& \times \oneh_{\sigma(i) < \sigma(j)} \oneh_{\sigma'(i) < \sigma'(j)} \,.
\end{split}
\end{equation}
The \emph{average Kendall kernel} is then derived as
\begin{equation}\label{eq:aveken}
\begin{split}
& K_W^{\avg}(\sigma, \sigma') := \frac{1}{n} \sum_{k=1}^n K_U^{\topk}(\sigma, \sigma') \\
& = \sum_{1 \leq i \neq j \leq n} \frac{1}{n} \min\cbr{\sigma(i), \sigma'(i)} \oneh_{\sigma(i) < \sigma(j)} \oneh_{\sigma'(i) < \sigma'(j)} \,.
\end{split}
\end{equation}
The positive definiteness of the average Kendall kernel is evident due to the fact that it is a sum of p.d. kernels, or we can verify that the average Kendall kernel is a weighted Kendall kernel of form (\ref{eq:kw}) whose weight function is indeed the $\min$ kernel \citep{Shawe-Taylor2004Kernel} thus satisfying the condition of Theorem~\ref{thm:pdk}. However, we see that the average Kendall kernel is no longer of the special form (\ref{eq:ku}) considered in Corollary~\ref{cor:pdk}.

\section{Fast computation}
Despite being a sum of $O(n^2)$ terms, the Kendall kernel (\ref{eq:kendall}) can be computed efficiently in $O(n \ln(n))$. For example, \citet{Knight1966computer} proposed such an algorithm by adapting a merge sort algorithm in order to count the inversion number of any permutation. While the general weighted Kendall kernel (\ref{eq:kw}) does not enjoy comparable computational efficiency in general, it does for certain choices discussed in Section~\ref{sec:wken}.
\begin{thm}\label{thm:comput}
The weighted Kendall kernel (\ref{eq:ku}) can be computed in $O(n \ln(n))$ for the top-$k$, additive or multiplicative weights. Besides, the average Kendall kernel (\ref{eq:aveken}) can also also be computed in $O(n \ln(n))$.
\end{thm}

\begin{proof}
The proof is constructive and the algorithm is summarized in Algorithm~\ref{alg:comput}. C++/R implementation available in the package \texttt{kernrank} at \url{https://github.com/YunlongJiao/kernrank}.

The algorithm can be decomposed into three parts. First, we compute $\pi := \sigma'\sigma^{-1}$ by carrying out the inverse and composition of permutations, which can be done in linear time (Line 1). Due to the right-invariance of any concerning kernel, we have $K(\sigma,\sigma') = \kappa(\pi)$ where $\kappa$ is the corresponding p.d. function:
\begin{equation*}
\begin{split}
\kappa_U^{\topk}(\pi) & = \sum_{1 \leq i < j \leq n} \oneh_{i \leq k} \oneh_{j \leq k} \oneh_{\pi(i) \leq k} \oneh_{\pi(j) \leq k} \oneh_{\pi(i) < \pi(j)} \,, \\
\kappa_U^{\add}(\pi) & = \sum_{1 \leq i < j \leq n} \br{u_i + u_j} \br{u_{\pi(i)} + u_{\pi(j)}} \oneh_{\pi(i) < \pi(j)} \,, \\
\kappa_U^{\mult}(\pi) & = \sum_{1\leq i < j \leq n} u_i u_j u_{\pi(i)} u_{\pi(j)} \oneh_{\pi(i) < \pi(j)} \,, \\
\kappa_W^{\avg}(\pi) & = \sum_{1\leq i < j \leq n} \frac{1}{n} \min\cbr{i, \pi(i)} \oneh_{\pi(i) < \pi(j)} \,.
\end{split}
\end{equation*}
Second, we register a global variable $s$ to record $\kappa(\pi)$ (Line 2) and implement $\kappa(\pi)$ in the function \textsc{QuickKappa} (Lines 3--36). Finally, $s$ is updated by calling the function \textsc{QuickKappa} (Line 37) and then outputted by the algorithm.

\begin{algorithm}[!p]
\caption{Top-$k$, average and weighted Kendall kernel with additive or multiplicative weight}
\label{alg:comput}
\begin{algorithmic}[1]
	\INPUT permutations $\sigma,\sigma'$, size $n$, $u$ for weighted Kendall kernel (optional), $k$ for top-$k$ Kendall kernel (optional)
	
	\STATE $\pi := \sigma' \sigma^{-1}$
	
	\STATE Initialize a global variable $s := 0$ then define
	\FUNCTION {\textsc{QuickKappa}$(indices)$}
		\IF {length of $indices > 1$}
			\STATE $pivot :=$ pick any element from $indices$
			\STATE $indhigh, indlow :=$ two empty arrays
			\STATE $cnum, ctop, cmin, cwa, cwb, cww := 0$
			\FOR {{\bfseries each} $i$ {\bfseries in} $indices$}
				\IF {$\pi(i) < \pi(pivot)$}
					\STATE Add $i$ to $indlow$
					\STATE $cnum \pluseq 1$
					\STATE $cmin \pluseq \min \{i, \pi(i)\}/n$
					\STATE $ctop \pluseq $ \IFTHENELSE {$i \leq k$ \AND $\pi(i) \leq k$} {$1$} {$0$}
					\STATE $cwa \pluseq u_i$
					\STATE $cwb \pluseq u_{\pi(i)}$
					\STATE $cww \pluseq u_i*u_{\pi(i)}$
				\ELSE
					\STATE Add $i$ to $indhigh$
					\SWITCH {type of weighted Kendall kernel}
						\CASE {STANDARD}
							\STATE $s \pluseq cnum$
						\ENDCASE
						\CASE {TOP-$k$}
							\STATE $s \pluseq $ \IFTHENELSE {$i \leq k$ \AND $\pi(i) \leq k$} {$ctop$} {$0$}
						\ENDCASE
						\CASE {AVERAGE}
							\STATE $s \pluseq cmin$
						\ENDCASE
						\CASE {ADDITIVE WEIGHT}
							\STATE $s \pluseq cww + cwa * u_{\pi(i)} + cwb * u_i + cnum * u_i * u_{\pi(i)}$
						\ENDCASE
						\CASE {MULTIPLICATIVE WEIGHT}
							\STATE $s \pluseq cww * u_i * u_{\pi(i)}$
						\ENDCASE
					\ENDSWITCH
				\ENDIF
			\ENDFOR
      \STATE \textsc{QuickKappa}$(indhigh)$
      \STATE \textsc{QuickKappa}$(indlow)$
		\ENDIF
	\ENDFUNCTION
	
	\STATE Call \textsc{QuickKappa}$([1,n])$ to update $s$
	
	\OUTPUT $K(\sigma,\sigma') = s$
\end{algorithmic}
\end{algorithm}

Central to the algorithm is the computation of $\kappa(\pi)$. It is based on an idea similar to a quicksort algorithm, where we recursively partition an array into two sub-arrays consisting of greater or smaller values according to a $pivot$, and cumulatively count the contributions between pairs of items with one in each sub-array. Specifically, suppose now $\pi$ is divided into two sub-arrays $\pi_{indhigh}$ and $\pi_{indlow}$ where ranks in $\pi_{indhigh}$ are all higher and those in $\pi_{indlow}$, now $\kappa(\pi)$ can be decomposed into
$$
\kappa(\pi) = \kappa(\pi_{indhigh}) + \kappa(\pi_{indlow}) + c(\pi_{indhigh}, \pi_{indlow}) \,,
$$
where $c$ characterizes the weighted non-inversion number of $\pi$ restricted on pairs of items with one in each sub-array. The computation of $c(\pi_{indhigh}, \pi_{indlow})$ depends on specific choice of weight and is depicted in the pseudo-code (Lines 19--30). Notably a single linear-time pass over $\pi$ is sufficient to compute $c(\pi_{indhigh}, \pi_{indlow})$. By the analysis of deduction typically for a quicksort algorithm, the overall time complexity of our algorithm is on average $O(n \ln (n))$.

In particular, recall that the standard Kendall kernel is merely a special case of the weighted Kendall kernel with constant weight, and hence our algorithm provides an alternative to the efficient algorithm based on merge sort proposed by \citet{Knight1966computer}. Notably, \citet{Vigna2015weighted} proposed a different fast algorithm based on the form of a weighted correlation (\ref{eq:vigna}), which applies to our additive and multiplicative cases as well.
\end{proof}

Theorem~\ref{thm:comput} shows that, just like the standard Kendall kernel, any of the weighted Kendall kernels concerned by the theorem benefits from the kernel trick and can be used efficiently by a kernel machine such as a SVM or kernel $k$-means.

\section{Learning the weights}\label{sec:learn}
So far, we have been focusing on studying the properties of the weighted Kendall kernel for which weights are given and fixed. In practice, it is usually not clear how to choose the weights so that the resulting kernel best suits the learning task at hand. We thus propose a systematic approach to \emph{learn} the weights in the context of supervised learning with discriminative models. More specifically, instead of first choosing \textit{a priori} the weights in the weighted Kendall kernel and then learning a function by a kernel machine, the weights can be learned jointly with the function estimated by the kernel machine.

We start by making explicit the feature embedding for permutations underlying the weighted Kendall kernel (\ref{eq:ku}). For any permutation $\sigma\in\Sn$, let $\Pi_{\sigma} \in \{0,1\}^{n \times n}$ be the permutation matrix of $\sigma \in \Sn$ defined by
$$
\br{\Pi_{\sigma}}_{ij} = \oneh_{i = \sigma(j)} \,.
$$
In fact, $\Pi: \Sn \to \RR^{n \times n}$ is the first-order permutation representation of $\Sn$ satisfying $\Pi_{\sigma}^\top = \Pi_{\sigma^{-1}}$ and $\Pi_{\sigma}\Pi_{\sigma'} = \Pi_{\sigma \sigma'}$ \citep{Diaconis1988Group}.
\begin{lem}\label{lem:phiu}
For any matrix of weights $U\in\RR^{n\times n}$, let $\Phi^U:\Sn \rightarrow \RR^{n\times n}$ be defined by
\begin{equation}\label{eq:phiu}
\begin{split}
\forall \sigma\in\Sn \,, \quad & \Phi^U(\sigma) = \Pi_\sigma^\top U \Pi_\sigma \\
\mbox{ or elementwise } & \br{\Phi^U(\sigma)}_{ij} = U_{\sigma(i),\sigma(j)} \,,
\end{split}
\end{equation}
and let kernel $G_U$ over $\Sn$ defined by
\begin{equation}\label{eq:gu}
\begin{split}
\forall \sigma,\sigma'\in\Sn \,, \quad & G_U(\sigma,\sigma') = \innerprod{\Phi^U(\sigma) , \Phi^U(\sigma')}_{\F} \\
& = \sum_{i,j=1}^n U_{\sigma(i),\sigma(j)} U_{\sigma'(i),\sigma'(j)} \,.
\end{split}
\end{equation}
where $\innerprod{\cdot,\cdot}_{\F}$ denotes the Frobenius inner product for matrices. In particular, if $U$ is upper (or lower) triangular with $U_{ij}=0$ for all $i \geq j$ (or $i \leq j$), then $G_U = K_U$.
\end{lem}

In the remainder of the paper, we will focus on $G_U$ (\ref{eq:gu}), which we call the \emph{weighted kernel} for permutations. Lemma~\ref{lem:phiu} shows that using the weighted kernel $G_U$ in a kernel machine amounts to learning in the feature space induced by the \emph{weighted embedding} $\Phi^U$. In the context of supervised learning with discriminative models, this reduces to fitting a linear model to data through embedding $\Phi^U$.
\begin{thm}\label{thm:suquan}
A general linear function on the weighted embedding $\Phi^U$ with coefficients $B\in\RR^{n\times n}$ can be written equivalently as
\begin{equation}\label{eq:hu}
\begin{split}
h^{U,B}(\sigma) & := \innerprod{B , \Phi^U(\sigma)}_{\F} \\ 
& = \innerprod{U , \Phi^B(\sigma^{-1})}_{\F} \\
& = \innerprod{ \vv(U) \otimes \br{\vv(B)}^\top , \Pi_\sigma \otimes \Pi_\sigma }_{\F} \,,
\end{split}
\end{equation}
where $\vv(\cdot)$ denotes the vectorization of a matrix, $\otimes$ denotes the Kronecker product for matrices or the outer product for vectors, wherever appropriate.
\end{thm}
\begin{proof}
The first equality is the definition of $h^{U,B}$. Plugging in (\ref{eq:phiu}) from Lemma~\ref{lem:phiu}, we have the second equality 
\begin{equation*}
\begin{split}
h^{U,B}(\sigma) & = \innerprod{B , \Phi^U(\sigma)}_{\F} \\
& = \tr\br{ B^\top \br{\Pi_\sigma^\top U \Pi_\sigma} } \\
& = \tr\br{ \br{\Pi_\sigma B \Pi_\sigma^\top}^\top U } \\
& = \tr\br{ \br{\Pi^\top_{\sigma^{-1}} B \Pi_{\sigma^{-1}}}^\top U } \\
& = \innerprod{U , \Phi^B(\sigma^{-1})}_{\F} \,.
\end{split}
\end{equation*}
The last equality follows from the relationship that $\vv(PXQ) = \br{Q^\top \otimes P} \vv(X)$ holds for any matrices $P,Q,X$. Note that the first two inner products are over $n\times n$ matrices, while the last one is over $n^2\times n^2$ matrices.
\end{proof}
\begin{rmk}
Theorem~\ref{thm:suquan} has several interesting implications. The first and second equalities show that the weights $U$ and linear coefficients $B$ are conceptually interchangeable. We can arbitrarily swap the roles of the coefficients $B$ of the linear model and the weights $U$ of the embedding, as long as we also change the representation of data from $\sigma$ to $\sigma^{-1}$. The last equality shows that a linear function with coefficients $B$ on an $n \times n$ dimensional embedding $\Phi^U(\sigma)$ is equivalent to a linear function with coefficients $\vv(U) \otimes \br{\vv(B)}^\top$ on an $n^2 \times n^2$ dimensional embedding $\Pi_\sigma \otimes \Pi_\sigma$.
\end{rmk}

Theorem~\ref{thm:suquan} suggests that, instead of first fixing some weights $U$ for the weighted kernel $G_U$ and then learning a function $B$ using this embedding, $U$ itself can be \emph{learned} jointly with $B$. Consequently, $\Phi^U$ with the learned weights $U$ is thus a data-driven discriminative feature embedding. Typically, fitting a kernel machine such as SVM is formulated as solving some optimization problem over $B$ (in case of fixed $G_U$). Now, we propose to optimize over both $U$ and $B$ by joint optimization, which amounts to a non-convex optimization problem over $\vv(U) \otimes \br{\vv(B)}^\top$ according to Theorem~\ref{thm:suquan}.

A commonly used approach to finding a stationary point of such optimization problems is by alternatively optimizing over $B$ for $U$ fixed, and over $U$ for $B$ fixed, until convergence or until a given number of iterations is reached. Interestingly, this can be easily implemented due to the conceptual interchangeability of $B$ and $U$ thanks to Theorem~\ref{thm:suquan}. Specifically, when $U$ is fixed, optimizing over $B$ amounts to training a kernel machine, with the kernel $G_U(\sigma,\sigma')$; in turn, for $B$ fixed, optimizing over $U$ also amounts to training a kernel machine, with the kernel $G_B(\sigma^{-1},(\sigma')^{-1})$ instead. Note that this naive alternating procedure implemented in kernel machines such as SVM, it will result in regularizing $B$ and $U$ in the same way, in the sense that the feasible region remain the same for both.

Alternatively, it is possible to bypass the need for alternative optimization of $B$ and $U$. Due to Theorem~\ref{thm:suquan}, when data are represented through $\Pi_\sigma \otimes \Pi_\sigma$, the linear coefficients $\vv(U) \otimes \br{\vv(B)}^\top$ can be directly learned by exerting some low-rank assumption on the linear model. This is similar to the \emph{supervised quantile normalisation} (SUQUAN) model proposed by \citet{LeMorvan2017Supervised}, who studied a similar optimization over a quantile distribution by learning a rank-1 linear model.

\section{High-order kernels for permutations}\label{sec:highorder}
While the Kendall kernel (\ref{eq:kendall}) relies only on pairwise comparison between items, it may be interesting in some applications to include high-order comparison among tuples beyond pairs~\cite{Diaconis1988Group}. In this section, we extend the Kendall kernel to higher-order kernels for permutations, establish their weighted variants and show that our systematic approach of learning the weights also applies to high-order cases.

Following the notation of (\ref{eq:order2}), let us denote for any $d \leq n$,
$$
\Ical{d} = \cbr{ (\idx{i}{,}) \in [1,n]^d \,:\, 1 \leq \idx{i}{<} \leq n} \,,
$$
and define the p.d. function associated with the order-$d$ kernel for permutations by
\begin{equation}\label{eq:orderd}
\kappa^{(d)}(\sigma) = \sum_{(\idx{i}{,}) \in \Ical{d}} \oneh_{ (\idpi{i}{\sigma}) \in \Ical{d} } \,.
\end{equation}
The \emph{order-$d$ kernel} is then given by the right-invariance, i.e.,
\begin{equation}\label{eq:pkd}
\begin{split}
& K^{(d)} (\sigma, \sigma') = \kappa^{(d)}(\sigma' \sigma^{-1}) \\ 
& = \sum_{\idx{i}{,}=1}^n \oneh_{(\idpi{i}{\sigma}) \in \Ical{d}}\oneh_{(\idpi{i}{\sigma'}) \in \Ical{d}} \,.
\end{split}
\end{equation}
The order-$d$ kernel compares two permutations by counting the number of $d$-tuples whose relative ordering is concordant. In particular, order-$2$ kernel reduces to the Kendall kernel (\ref{eq:kendall}), i.e., $K^{(2)} = K_\tau$.

Similarly to (\ref{eq:phiu}) and (\ref{eq:gu}), we define a \emph{weighted order-$d$ embedding} for permutations $\Phi^{\UU} : \Sn \to \RR^{\tsorder{d}}$ element-wise by
\begin{equation}\label{eq:phiu-tensor}
\forall \sigma\in\Sn \,, \quad \br{\Phi^{\UU}(\sigma)}_{\idx{i}{}} = \UU_{\idpi{i}{\sigma}} \,,
\end{equation}
where the weights $\UU \in \RR^{\tsorder{d}}$ is a order-$d$ (cubical) tensor of size $n$ at each dimension. Notably, given $\UU$, computation of $\Phi^{\UU}(\sigma)$ for any $\sigma \in \Sn$ can be done simply by permuting the entries in each dimension of the tensor $\UU$ by $\sigma$. The \emph{weighted order-$d$ kernel} is then defined as
\begin{equation*}
\begin{split}
\forall \sigma,\sigma'\in\Sn \,, \quad & G_{\UU}(\sigma,\sigma') = \innerprod{\Phi^{\UU}(\sigma) , \Phi^{\UU}(\sigma')}_{\F} \\
& = \sum_{\idx{i}{,}=1}^n \UU_{\idpi{i}{\sigma}} \UU_{\idpi{i}{\sigma'}} \,,
\end{split}
\end{equation*}
where $\innerprod{\cdot,\cdot}_{\F}$ denotes the Frobenius inner product for tensors. In particular, when $\UU$ is a matrix (order-2 tensor), the weighted order-$2$ kernel reduces to the weighted kernel (\ref{eq:gu}). Taking the special case ${\UU}_{\idx{i}{}} = \oneh_{\idx{i}{<}}$, the weighted order-$d$ kernel $G_{\UU}$ reduced to the standard order-$d$ kernel $K^{(d)}$ in (\ref{eq:pkd}).

Again, we would like to learn the weights $\UU$ adapted to data by jointly optimizing them with the function estimated by a kernel machine. Recall that, in the context of supervised learning with discriminative models, working with the weighted order-$d$ kernel $G_{\UU}$ for permutations with kernel machines amounts to fitting a linear model to data through the weighted order-$d$ embedding $\Phi^{\UU}$. In case of weighted high-order embedding and kernel, we can establish high-order counterpart of the results in Theorem~\ref{thm:suquan}.
\begin{thm}\label{thm:suquan-d}
A general linear function on the weighted order-$d$ embedding $\Phi^{\UU}$ with some order-$d$ tensor of linear coefficients $\BB\in\RR^{\tsorder{d}}$ can be written equivalently as
\begin{equation}\label{eq:hu-d}
\begin{split}
h^{\UU,\BB}(\sigma) & := \innerprod{\BB , \Phi^{\UU}(\sigma)}_{\F} \\ 
& = \innerprod{\UU , \Phi^{\BB}(\sigma^{-1})}_{\F} \\
& = \innerprod{ \UU \otimes \BB , \Pi_\sigma^{\otimes d} }_{\F} \,,
\end{split}
\end{equation}
where $\otimes$ denotes the outer product\footnote{With respect to standard index order (consecutively concatenating the dimensions of each tensor).} for tensors, $\Pi_\sigma^{\otimes d} \in (\{0,1\})^{\tsorder{2d}}$ is defined element-wise by
\begin{equation*}
\begin{split}
\br{ \Pi_\sigma^{\otimes d} }_{\idx{i}{} \idx{j}{}} & = \prod_{k=1}^d \br{\Pix}_{i_k j_k} \\
& = \oneh_{(\idx{i}{,}) = \pix((\idx{j}{,}))} \,.
\end{split}
\end{equation*}
\end{thm}
\begin{proof}
By definition, the weighted order-$d$ embedding $\Phi^{\UU} (\sigma)$ (\ref{eq:phiu-tensor}) is the contracted product \citep[Section 3.3]{Bader2006Algorithm} of two tensors $\Pi_\sigma^{\otimes d} \in (\{0,1\})^{\tsorder{2d}}$ and $\UU \in \RR^{\tsorder{d}}$ respectively at dimensions $\tsprod$ of both tensors, resulting in an order-$d$ tensor (inheriting the remaining indices of the first tensor). Specifically, we can write $\Phi^{\UU}$ in a compact representation
\begin{equation*}
\begin{split}
\Phi^{\UU} (\sigma) & = \innerprod{ \T , \UU }_{\tsprod} \\
& = \br{ \UU_{\idpi{i}{\sigma}} }_{\idx{i}{}} \in \RR^{\tsorder{d}} \,,
\end{split}
\end{equation*}
where $\innerprod{ \cdot, \cdot }_{\tsprod}$ denotes the contracted product respectively at dimensions $\tsprod$.

Linear function are now of the form:
\begin{equation*}
\begin{split}
h^{\UU, \BB} (\sigma) & = \innerprod{ \BB, \Phi^{\UU} (\sigma) }_{\F} \\
& = \innerprod{ \BB, \innerprod{ \T , \UU }_{\tsprod} }_{\F} \\
& = \sum_{\idx{i}{,} = 1}^n \BB_{\idx{i}{}} \UU_{\idpi{i}{\sigma}} \\
& = \sum_{\idx{i}{,} = 1}^n \UU_{\idx{i}{}} \BB_{\idpi{i}{\sigma^{-1}}} \\
& = \innerprod{ \UU, \innerprod{ \T , \BB }_{\tsprodalt} }_{\F} \\
& = \innerprod{ \UU, \Phi^{\BB} (\sigma^{-1}) }_{\F} \\
& = \innerprod{ \UU \otimes \BB , \T }_{\F} \, ,
\end{split}
\end{equation*}
where $\innerprod{ \cdot, \cdot }_{\F}$ denotes the Frobenius inner product for tensors, whereas $\innerprod{ \cdot, \cdot }_{\tsprodalt}$ denotes the contracted product respectively at dimensions $\tsprodalt$.
\end{proof}

Theorem~\ref{thm:suquan-d} together with Theorem~\ref{thm:suquan} ensures that there exists no substantial difference when we move from order-2 to order-$d$ embedding, or their underlying weighted kernel. Therefore, the discussion elaborated in Section~\ref{sec:learn} regarding Theorem~\ref{thm:suquan} can be seamlessly migrated to higher-order cases. For instance, as joint optimization over $\BB$ and $\UU$ is non-convex, alternative optimization can be employed to find a stationary point in practice. Thanks to the conceptual interchangeability between $\BB$ and $\UU$, alternating the optimization in $\BB$ and $\UU$ with a kernel machines amounts to simply changing the underlying kernel function accordingly. Note that the square unfolding studied by \citet[Definition 2.2]{Jiang2015New} of the order-$2d$ tensor $\Pi_\sigma^{\otimes d}$ is a matrix of dimension $n^d \times n^d$, which is exactly the $d$-fold Kronecker product of the matrix $\Pi_\sigma$ with itself. Therefore, we can still follow approaches developed by \citet{LeMorvan2017Supervised} in order to directly learn $\UU \otimes \BB$ by asserting low-rank assumptions on a linear model.

\section{Numerical experiments}\label{sec:experiment}
In this section, we demonstrate the use of the proposed weighted kernels compared with the standard Kendall kernel for classification on a real dataset from the European Union survey \textit{Eurobarometer 55.2} \citep{Christensen2010Eurobarometer}. As part of this survey, participants were asked to rank, according to their opinion, the importance of six sources of information regarding scientific developments: TV, radio, newspapers and magazines, scientific magazines, the internet, school/university. The dataset also includes demographic information of the participants such as gender, nationality or age. We removed all respondents who did not provide a complete ranking over all six sources, leaving a total of $12,216$ participants. Then, we split the dataset across age groups, where $5,985$ participants were 40 years old or younger, $6,231$ were over 40. The objective is to predict the age group of participants from their ranking of $n=6$ sources of news. In order to compare different kernels, we chose to fit kernel SVMs with different kernels.

In order to assess the classification performance, we randomly sub-sampled a training set of $400$ participants and a test set of $100$ participants. This random sub-sampling process was repeated $50$ times and we report the classification accuracy on $50$ test sets. Different types of weighted kernels considered are: the weighted Kendall kernels (\ref{eq:ku}) with top-$k$ weight for $k = 2,\dots,6$, additive weight for hyperbolic or logarithmic reduction factor, multiplicative weight for hyperbolic or logarithmic reduction factor, the average Kendall kernel (\ref{eq:aveken}), the weighted kernel (\ref{eq:gu}) with learned weights via alternative optimization or via SUQUAN-SVD, an SVD-based low-rank approximation algorithm proposed by \citet[Algorithm 1]{LeMorvan2017Supervised}. In particular, the top-$6$ Kendall kernel is equivalent to the standard Kendall kernel, which is seen as the baseline to benchmark with.

\begin{table}[t]
\caption{Classification accuracy across $50$ random experiments using kernel SVM with different kernels (ordered by decreasing mean accuracy). $p$-value indicates the significance level of a kernel \textit{better} than the standard Kendall kernel.}
\label{tab:acc}
\begin{center}
\begin{small}
\begin{sc}
\begin{tabular}{lcc}
\toprule
Type of weighted kernel & Mean $\pm$ SD & $p$-value \\
\midrule
average & 0.641 $\pm$ 0.058 & 0.01 \\
learned weight (svd) & 0.635 $\pm$ 0.051 & 0.07 \\
top-4 & 0.632 $\pm$ 0.053 & 0.31 \\
top-5 & 0.631 $\pm$ 0.050 & 0.20 \\
add weight (hb) & 0.628 $\pm$ 0.048 & 0.59 \\
standard (or top-6) & 0.627 $\pm$ 0.050 & --- \\
add weight (log) & 0.627 $\pm$ 0.050 & --- \\
learned weight (opt) & 0.626 $\pm$ 0.051 & --- \\
top-3 & 0.624 $\pm$ 0.044 & --- \\
mult weight (hb) & 0.623 $\pm$ 0.048 & --- \\
mult weight (log) & 0.618 $\pm$ 0.048 & --- \\
top-2 & 0.590 $\pm$ 0.051 & --- \\
\bottomrule
\end{tabular}
\end{sc}
\end{small}
\end{center}
\vskip -0.3in
\end{table}

Table~\ref{tab:acc} and Figure~\ref{fig:acc} summarize the classification accuracy across $50$ random experiments using kernel SVM with different types of weighted kernels. Results show that the average Kendall kernel is the best-performing kernel of all the kernels considered, followed by the weighted kernel with weights learned via SUQUAN-SVD. Since we observed that the standard deviation (SD) of accuracy scores is relatively large compared to the mean, we carried out paired Wilcoxon rank test to assess whether or not the better performance is significant where certain weighted kernels improved the standard Kendall kernel. The $p$-values are reported in the last column of Table~\ref{tab:acc}. In fact, the average Kendall outperformed the standard Kendall kernel at a significance of $0.01$ whereas the weighted kernel with weights learned via SUQUAN-SVD has a significance of $0.07$. However, compared to to SUQUAN-SVD, the weighted kernel with weights learned via alternative optimization is particularly not a good choice, which may be due to the fact that fully optimizing over the weights under a rather simple ridge-type regularization can overfit the data, leading to poor generalization. In conclusion, these results show that the weighted Kendall kernel can be promising for certain prediction tasks as they can outperform the standard one.

It is interesting to make a note on the performance of top-$k$ Kendall kernels varying the cutoff parameter $k$ and the average Kendall kernel (i.e. boxplots to the left of the first vertical dotted line in Figure~\ref{fig:acc}). When we move away from top-6 with all pairs of items included (or equivalently the baseline standard Kendall kernel), as $k$ decreases, the performance increases until peaking at $k = 3$, then as $k$ continues to decrease, the performance decreases until $k = 2$ with only the top ranked pair of items considered. This implies that choosing the number of top ranked items should be crucial in learning with rankings. Further, the performance can be drastically improved when we use the average Kendall kernel which crowdsourced all top-$k$ Kendall kernels.

\begin{figure}[t]
\begin{center}
\centerline{\includegraphics[width=0.8\columnwidth]{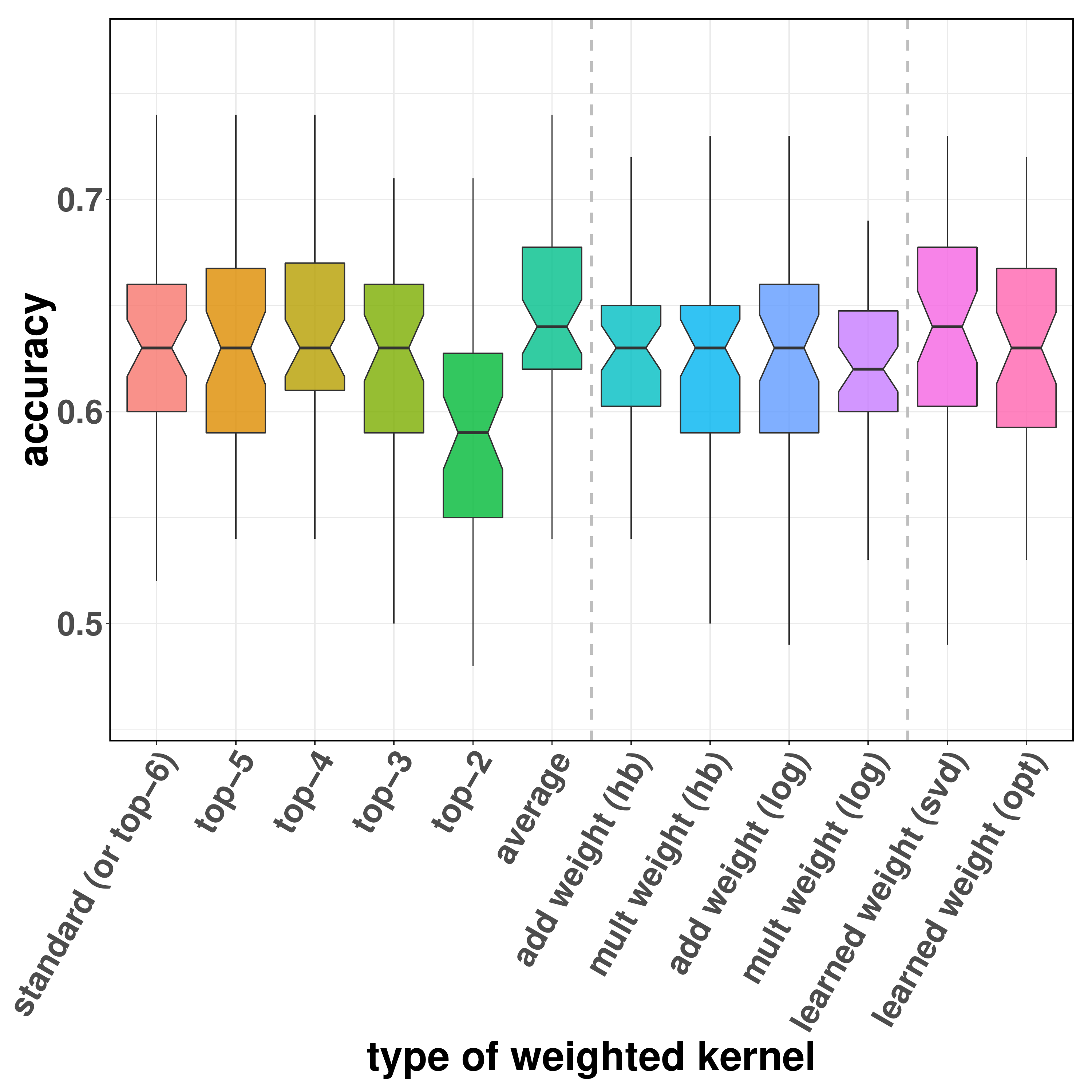}}
\caption{Boxplot of classification accuracy across $50$ random experiments using kernel SVM with different kernels.}
\label{fig:acc}
\end{center}
\vskip -0.3in
\end{figure}

\begin{figure}[t]
\begin{center}
\centerline{\includegraphics[width=\columnwidth]{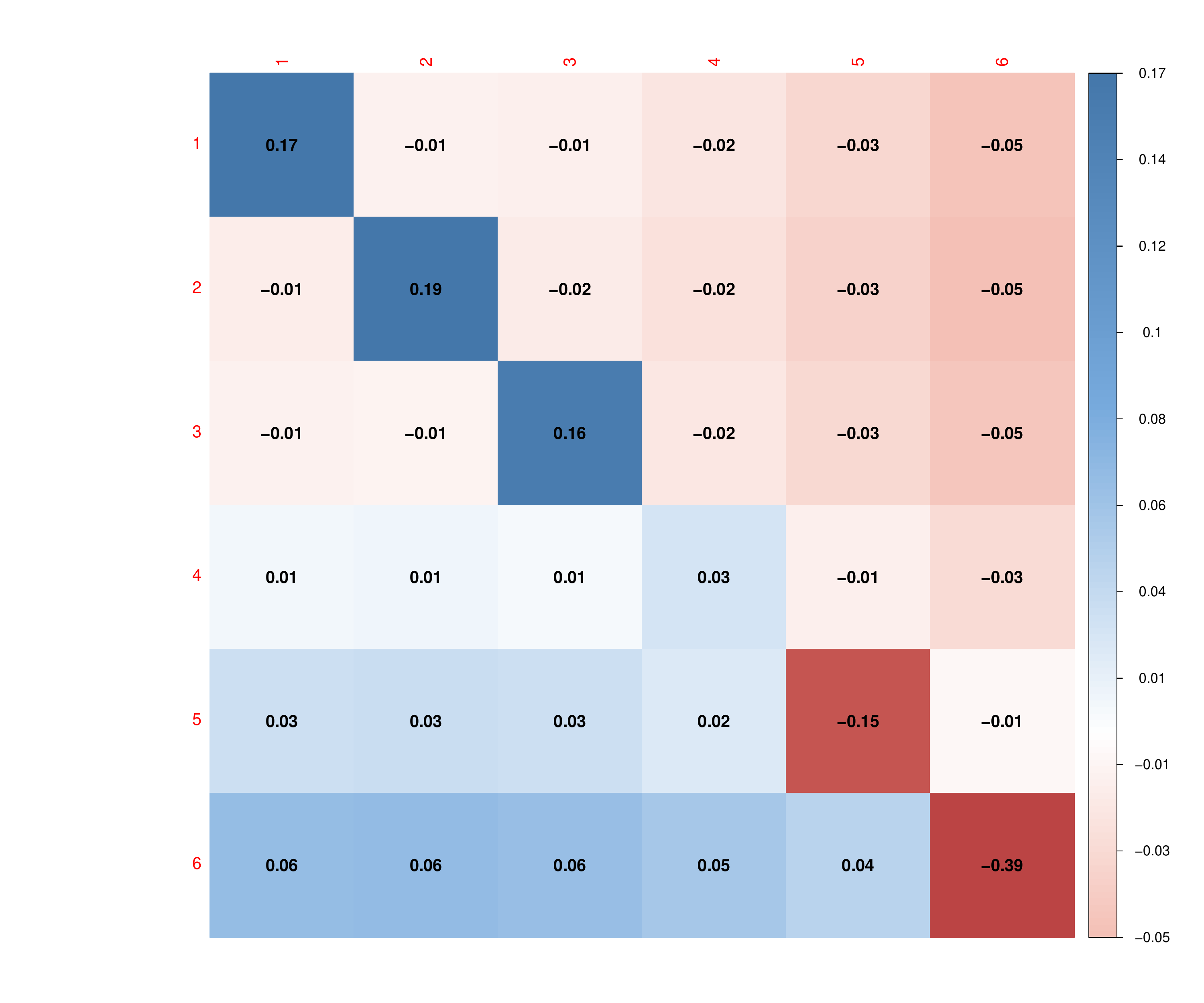}}
\caption{Weights learned via SUQUAN-SVD \citep[Algorithm 1]{LeMorvan2017Supervised} for a data-driven weighted kernel.}
\label{fig:weights}
\end{center}
\vskip -0.3in
\end{figure}

Finally, we show the weights learned via SUQUAN-SVD for a data-driven weighted kernel in Figure~\ref{fig:weights}. Recall that the entries $U_{ab}$ of a weight matrix encodes the importance of pairs of ranked positions $a$ and $b$. First, the diagonal can be seen as importance of individual positions, and we observe a generally decreasing trend of importance as the position rank increases, i.e., as items become less preferred. This is in line with the pattern widely recognized in such data involved in preference ranking. Second, the matrix presents a near skew-symmetric pattern, where the upper triangle contains mostly negative importance and the lower triangle mostly positive. This suggests that shifting the relative order of $\sigma(i) > \sigma(j)$ or $\sigma(i) < \sigma(j)$ results in opposite direction of contribution in evaluating the kernel $G_{U}$ (\ref{eq:gu}). Finally, we observe that, the more we move away from diagonal, the larger the magnitude of the values, which indicates that it is crucial to focus more on pairs of items whose ranks are more distinctively placed by a permutation. This pattern of gradient on pairwise position importance in $U$ identifies our foremost motivation of proposing the weighted kernels.

\section{Discussion}
We have proposed a general framework to build computationally efficient kernels for permutations, extending the Kendall kernel to incorporate weights and higher-order information, and showing how weights can be systematically optimized in a data-driven way. The price to pay for these extensions is that the dimensionality of the embedding and the number of free parameters of the model can quickly increase, raising computational and statistical challenges that could be addressed in future work. On the theoretical side, the kernels we proposed could naturally be analyzed in the Fourier domain \citep{Kondor2010Ranking, Mania2016Universality} which may lead to new interpretations of their regularization properties and reproducing kernel Hilbert spaces \citep{Berg1984Harmonic}.

\bibliography{mybib}
\bibliographystyle{plainnat}

\end{document}